\documentclass[6pt,twocolumn]{article}

\usepackage{amsmath,amssymb}
\usepackage{amsthm}
\usepackage{graphicx}
\usepackage[numbers]{natbib}
\usepackage{subcaption}
\usepackage[normalem]{ulem}
\usepackage{xcolor}
\usepackage{algorithm}
\usepackage{algpseudocode}

\newtheorem{definition}{Definition}[section]

\newtheorem{theorem}{Theorem}[section]
\newtheorem{lemma}{Lemma}[section]
\newtheorem{proposition}{Proposition}[section]
\newtheorem*{remark}{Remark}
\usepackage{url}
\usepackage{hyperref}

\title{State-Conditional Adversarial Learning:\\
An Off-Policy Visual Domain Transfer Method for End-to-End Imitation Learning%
\thanks{Code available at \url{https://github.com/Xiang-Foothill/BkgGeneralizor.git}}}

\author{
  Yuxiang Liu\textsuperscript{\(\dagger,*\)}
  \and
  Shengfan Cao\textsuperscript{\(\dagger,*\)}
}

\date{University of California, Berkeley, CA, USA}

\begin{document}
\maketitle

\begingroup
\renewcommand{\thefootnote}{}
\footnotetext{\textsuperscript{\(\dagger\)}These authors contributed equally.}
\footnotetext{\textsuperscript{*}Corresponding authors: \url{liu.yx@berkeley.edu} and \url{shengfan_cao@berkeley.edu}.}
\endgroup

\begin{abstract}                
We study visual domain transfer for end-to-end imitation learning in a realistic and challenging setting where target-domain data are off-policy, expert-free, and scarce. We first provide a theoretical analysis showing that the target-domain imitation loss can be upper bounded by the source-domain loss plus a state-conditional latent KL divergence between source and target observation models. Guided by this analysis, we propose State-Conditional Adversarial Learning (SCAL), an off-policy adversarial framework that aligns latent distributions conditioned on system state using a discriminator-based estimator of the conditional KL term. Experiments on visually diverse autonomous driving environments built on CARLA demonstrate that SCAL achieves robust transfer and strong sample efficiency.
\end{abstract}


\section{Introduction}

Vision-based imitation learning (IL) has achieved impressive results across diverse robotic domains, including aerial drones \citep{xing2024bootstrapping, pfeiffer2022attention}, autonomous driving \citep{bojarski2016end}, and manipulation \citep{rahmatizadeh2017vision, zhu2022viola}. By learning directly from high-dimensional visual observations, these methods avoid reliance on specialized sensors while enabling agents to reproduce complex expert behaviors. However, despite their empirical success, vision-based IL policies remain brittle when deployed in visual domains insufficiently represented in the training distribution.

Existing approaches for improving generalization in vision-based IL can be broadly categorized into zero-shot and few-shot adaptation methods. Zero-shot techniques such as domain randomization \citep{tobin2017domain, mehta2020active}, DARLA \citep{darla2017}, and DARL \citep{li2021domain} attempt to transfer policies without any access to target-domain data, relying on the assumption that synthetic variations can capture target-domain visual characteristics. Few-shot methods, by contrast, leverage limited target-domain data \citep{bewley2018learning, invariance_latent_alignment, stadie2017thirdperson, fewshot_transfer_bc, hansen2021self}, but typically impose strong assumptions on what information is available. Several works assume the agent may execute on-policy rollouts in the target domain to gather online data \citep{invariance_latent_alignment, stadie2017thirdperson, hansen2021self}. Others assume access to target-domain expert demonstrations \citep{domain_adaptive_imitation}. Although \citep{bewley2018learning} removes the need for both rollouts and expert demonstrations, its CycleGAN-based pixel translation requires large unlabeled target datasets, making it unsuitable for data-scarce settings.

In many real-world scenarios, these assumptions do not hold. Target environments are often safety-critical or operationally expensive, making on-policy exploration costly or infeasible. Human or controller-based demonstrations may be difficult to obtain due to labor constraints, hardware wear, or lack of reliable expert solutions. Even off-policy data collection is limited by hardware availability, mission time, or regulatory restrictions. Consequently, target-domain data in these practical settings are commonly subject to the following constraints: \textbf{a lack of expert supervision}, \textbf{off-policy distribution}, and \textbf{scarcity}. In this paper, we aim to overcome these limitations with a unified framework. We establish theoretical guarantees for operating under off-policy and expert-free target-domain assumptions, and we empirically validate the framework's high data efficiency to address data scarcity. 
To the best of our knowledge, this is the first methodology to jointly resolve all three constraints.


The main contribution of this paper is as follows: 

\begin{itemize}
\item We provide a formal upper-bound analysis for domain-transfer problems.
\item We introduce State-Conditional Adversarial Learning (SCAL), a novel off-policy, expert-free domain-transfer framework.
\item We empirically validate our approach's sample efficiency and theoretical validity in vision-based driving tasks.
\end{itemize}

\section{Problem Formulation}
\subsection{Definition and Control Objective}
Consider the following class of non-linear, time-invariant, deterministic, discrete-time system with stochastic observations:
\begin{equation}\label{eq:system}
\begin{aligned}
& x_{k + 1} = f(x_k, u_k), \quad y_k \sim e(\cdot \mid x_k), ~~~ x_0 \sim \mathcal{X}_0, \\
& x_k \in \mathcal{X}, ~~~u_k \in \mathcal{U}, ~~~ y_k \in \mathcal{Y}, ~~~ \forall k \in \mathbb{N}. 
\end{aligned}
\end{equation}
where $x_k, u_k$ are the state and input at time $k$; $f$ is the dynamics, which is assumed to be known; $y_k$ is the observation at time $k$, which follows the unknown state-dependent observation distribution $e(\cdot \mid x_k)$; $\mathcal{X}_0$ is the initial state distribution; $\mathcal{X}$, $\mathcal{U}$, and $\mathcal{Y}$ are the state space, the action space, and the observation space, respectively. 


Let $\pi_\theta: \mathcal{Y} \mapsto \mathcal{U}$ be a parametric observation feedback policy. Following standard end-to-end visuomotor imitation learning frameworks ~\citep{end2end_visuomotor, bojarski2016end, invariance_latent_alignment}, we parameterize the policy as
$\pi_\theta(y_k) = D_w(E_\phi(y_k))$,  
where $E_\phi$ is a parametric encoder network, $D_w$ is the parametric control head, and 
\begin{equation*}
    l_k = E_\phi(y_k) \in \mathcal{L}
\end{equation*}
is the latent representation of the observation $y_k$, and $\mathcal{L}$ denotes the latent space. 
$\theta = \begin{bmatrix}w & \phi\end{bmatrix}$ is the collection of all parameters of the policy. 

The interaction between any agent $\pi$ and system ~\eqref{eq:system} can be viewed as a Markov decision process (MDP). Let $p^k(\cdot \mid e, \pi)$ be state distribution of such MDP at k-th step.

\bigskip
\begin{definition}
[Discounted Visitation Distribution \citep{ho2016generative, kakade2002approximately}] \label{def:discounted-visitation-distribution}
Let $\gamma \in (0, 1)$ be a discount factor. 
The discounted state visitation distribution induced by policy $\pi$ is defined as: 
\begin{equation}
    p(x \mid e, \pi) \triangleq (1 - \gamma) \sum_{k=0}^{\infty} \gamma^k\, p^k(x_k = x \mid e, \pi).
\end{equation}

For a given encoder $E$, the latent distribution conditioned at a given state $x$ is:
\begin{equation}
\begin{aligned}
p(l \mid x, e, E) = \mathbb{P} \big( l = E(y) \; \mid \; y \sim e(\cdot \mid x)\big) 
\end{aligned}
\end{equation}
We will refer to the above distribution a state-conditioned latent distribution.
\end{definition}


Suppose we have two systems sharing the same known dynamics $f$ but distinct observation models $e_s$ and $e_t$, where $e_s$ is known and $e_t$ is unknown. We refer to the $e_s$ as the \textbf{source domain}, and $e_t$ as the \textbf{target domain}. 
As a shorthand, the domain-specific discounted visitation distributions are denoted as $p_s(\cdot \mid \pi) = p(\cdot \mid e_s, \pi)$ and $p_t(\cdot \mid \pi) = p(\cdot \mid e_t, \pi)$ for the source domain and the target domain, respectively. Analogously, the same short-hand applies for the state-conditioned latent distribution $p_s(l \mid x, E) = p(l \mid x, e_s, E)$ and $p_t(l \mid x, E) = p(l \mid x, e_t, E)$.
For a given domains $e_i, \; i \in \{ s, t\}$, we define the cost as
\begin{equation}
\begin{split}
J(\pi; e_i) &\triangleq \sum_{k=0}^{\infty} \mathbb{E} \left[ c\bigl(x_k, \pi(y_k)\bigr)\right], \\
&\quad x_k \sim p^k(\cdot \mid e_i, \pi), \quad y_k \sim e_i(\cdot \mid x_k).
\end{split}
\end{equation}
subject to the constraints and distributions in~\eqref{eq:system}. $c(\cdot, \cdot)$ is a non-negative cost function. Let $c(x_k, u_k) = \infty$ if $x_k \notin \mathcal{X}$ or $u_k \notin \mathcal{U}$.

The overall optimization objective is:
\begin{equation} \label{control_objective}
\min_{\theta} \frac{1}{2} \bigl(J(\pi_\theta; e_s) + J(\pi_\theta; e_t)\bigr)
\end{equation}

\subsection{Supervision}

We assume a high-performing black-box expert $\pi_{\beta}$ which provides supervision $u^\star_k$ at a given state $x_k$, 
$$u^{\star}_k = \pi_\beta(x_k, h_k),$$
where $h_k = h(\mathbf{x}_{0:k-1})$ is the hidden state of the expert, which contains information of the closed-loop trajectory prior to time $k$, denoted as $\mathbf{x}_{0:k-1}$. 

Since the dynamics $f$ and the source domain observation distribution $e_s$ is known, we can follow the learning framework in DAgger \citep{ross11noregret}, an online learning framework, to collect a dataset $\mathcal{B}_s = \{(y_k, x_k, u^\star_k)\}$ with asymptotically no covariate shift. 

Note that such direct data collection is impossible for the target domain because $e_t$ is unknown. 
However, we assume access to a small dataset $\mathcal{B}_t=\{(y_k, x_k)\}$ with observation-state pairs. 
The dataset does not necessarily follow a closed-loop trajectory, so we cannot acquire expert supervision because of the absence of hidden states. 
In addition, even if supervision is available, vanilla imitation learning framework still suffers from the covariate shift. 




\subsection{Optimal Control via Imitation Learning}
With the expert supervision $u^\star_k$, we can reformulate the objective ~\eqref{control_objective} in the following imitation loss form: 
\begin{equation}
\label{imitation_learning_surrogate}
\min_{\theta} \; \mathcal{J}_s(\theta) + \mathcal{J}_t(\theta).
\end{equation}
where
\begin{equation}
\mathcal{J}_i(\theta) = \mathbb{E}_{(y, u^\star) \sim p(\cdot \mid e_i, \pi_\theta)}[d(\pi_{\theta}(y), u^\star)], \quad i\in \{s, t\}
\end{equation}
Function $d : \mathcal{U} \times \mathcal{U} \rightarrow \mathbb{R}$ measures action difference. $\mathcal{J}_s(\theta)$ and $\mathcal{J}_t(\theta)$ are imitation surrogates for $J(\pi_\theta ; e_s), \; J(\pi_\theta; e_t)$ in ~\eqref{control_objective} respectively.


Note that~\eqref{imitation_learning_surrogate} cannot be directly solved under imitation learning frameworks because $e_t$ is unknown, and consequently, we lack the data to estimate $\mathcal{J}_t(\theta)$ directly.  
In Section~\ref{upper_bound_analysis}-\ref{sec:proposed-approach}, we leverage adversarial learning to provide an upper bound for $\mathcal{J}_t(\theta)$ as its surrogate, and in Section~\ref{sec:experiments}, we empirically show the validity of this approach. 


\section{Related Works}

\subsection{Imitation Learning}
Imitation Learning has seen great success in recent years e.g., \citep{bojarski2016end, xing2024bootstrapping}. One important challenge is distributional shift: the trajectory distribution induced by the agent during inference time is not consistent with the trajectory of the expert from the data buffer. DAgger-style framework in \citep{ross11noregret, ZhangC17, RossB14} solves this problem by mixing agents' actions with expert actions when collecting data. In this paper, we leverage DAgger as the base learning pipeline in the source-domain.

\subsection{Adversarial Learning}
Adversarial learning is widely used for distribution alignment by the community of computer vision e.g., \citep{arjovsky2017wasserstein, dann2016, gan_theory_survey2018}. Built upon GAN-style framework, conditional adversarial learning is further introduced to align conditional distributions e.g., \citep{cgan_dann, cgan_projection, mirza2014conditional} to further tackle the multi-modal cases. In this work, we followed the paradigm of conditional adversarial learning to accomplish the matching between latent distributions conditioned on task-relevant information.

\section{Transfer Learning via Alignment} \label{upper_bound_analysis}

The key idea is that while $\mathcal{B}_t$ only contains off-policy data without supervision, 
we can leverage $\mathcal{B}_s$ and $\mathcal{B}_t$ to establish the alignment between the two observation distributions as a means of transfer learning. 
In this section, we formally define alignment, and analyze its connection to imitation loss. 

\subsection{Alignment and Connection to Imitation Learning}

\begin{definition}[Alignment]\label{def:alignment} 

Recall the definition for state-conditioned latent distribution, $p_s(l \mid x, E_{\phi})$ = $\mathbb{P}(l = E_{\phi}(y) \; \mid \; y \sim e_s( \cdot \mid x))$. Similarly for $p_t(l \mid x, E_{\phi})$. 
If the encoder network $E_\phi$ induces identical state-conditioned latent distribution under the two observation distributions $e_s$ and $e_t$ for all possible states $x$ i.e., 

$$p_s(l \mid x, E_\phi) = p_t(l \mid x,  E_\phi), \; \forall x \in \mathcal{X}, \; l \in \mathcal{L}$$


then $e_s$ and $e_t$ are \textbf{aligned} by $E_\phi$. 
\end{definition}

\begin{lemma}\label{lemma_one}
If the policy $\pi_\theta$ has an encoder $E_\phi$ that aligns $e_s$ and $e_t$, then
\[
\mathcal{J}_t(\theta) = \mathcal{J}_s(\theta).
\]
\end{lemma}
\begin{proof}
See Appendix \ref{lemma_one}.
\end{proof}
Lemma~\ref{lemma_one} implies that under perfect alignment, the performance of $\pi_\theta$ in the source domain exactly matches that in the target domain.


Perfect alignment is challenging to attain, especially with a limited $\mathcal{B}_t$.  
We quantify the alignment loss of between $e_s$ and $e_t$ as the expected Kullback–Leibler divergence between the distribution of the latent encoded by $E_\phi$. Formally, 
\begin{equation}\label{eq:alignment-loss}
    L(\phi) = \mathbb{E}_{p_s(x \mid \pi_\theta)} \left[d_{\mathrm{KL}} (p_s(l \mid x, E_{\phi}) \Vert p_t(l \mid x, E_{\phi})\right],
\end{equation}
where $d_{\mathrm{KL}}(\cdot \,\|\, \cdot)$ denotes the Kullback–Leibler divergence. For the rest of the paper, we will refer quantity as alignment loss. 

In general, the connection between alignment and imitation loss is as follows. 

\begin{theorem} \label{theorem_one}
For a policy $\pi_\theta$ with visual encoder $E_\phi$, its target-domain imitation loss $\mathcal{J}_{t}(\theta)$ can be upper bounded by
\begin{equation} \label{eq:upper_bound}
\begin{aligned}
&\mathcal{J}_t(\theta) \;\leq\; 
\mathcal{J}_s(\theta)
+ \alpha \sqrt{ 
    \frac{2\gamma}{1 - \gamma} \, (
    L(\phi)
    + \sigma )}~,
\end{aligned}
\end{equation}

where
\begin{itemize}
    \item $\sigma = d_{\mathrm{KL}}\!\big(e_s(\cdot \mid x)\,\|\, e_t(\cdot \mid x)\big) \ge 0$,
    \item $
    \alpha = \sup_{y \in \mathcal{Y}, \; u^{*} \in \mathcal{U}}
    d\big(\pi_{\theta}(y), u^{*} \big) \ge 0
    $
    is the uniform bound over the loss function,
    \item $\mathcal{J}_s(\theta), \mathcal{J}_t(\theta)$ are source-domain and target-domain imitation imitation losses respectively,
    \item $L(\phi)$ is the alignment loss,
    \item $\gamma$ is the discount factor per definition \ref{def:discounted-visitation-distribution}.

\end{itemize}
\end{theorem}

\begin{proof}
See Appendix \ref{theorem_one}.
\end{proof}

\begin{remark}
The term $\sigma$ bounds the inherent divergence between the source and target observation models. While this represents an irreducible constant, its magnitude can be kept reasonably small in practice via robust computer vision preprocessing, such as input normalization. See appendix  \ref{sec:remark_justification} for the mathematical justification.
\end{remark}


Theorem~\ref{theorem_one} shows that the target domain imitation loss $\mathcal{J}_t(\theta)$ can be optimized by minimizing the source domain imitation loss $\mathcal{J}_s$ and the alignment loss $L(\phi)$, without access to on-policy data in the target domain.


\begin{proposition} \label{KL-divergence surrogate proposition}
For the sake of Majorization–Min, we propose the following as a surrogate for the joint objective in~\eqref{imitation_learning_surrogate}:

\begin{equation}
\min_{\theta}\;
\mathcal{J}_s(\theta) \;+\;
L(\phi).
\end{equation}
\end{proposition}

\section{Proposed Approach}\label{sec:proposed-approach}
In this section, we propose minimizing the alignment loss $L(\phi)$ in~\eqref{eq:alignment-loss} via adversarial learning, and present State-Conditional Adversarial Learning (SCAL), which is an augmentation of traditional imitation learning frameworks for transfer learning through alignment. 

\subsection{Discriminator-based Off-policy Evaluation}
In practice, we estimate it by training a discriminator $Q_{\psi}$ parameterized with $\psi$ to distinguish between $(l, x)$ pairs sampled from $\mathcal{B}_s$ and $\mathcal{B}_t$. Note that we can sample $(l, x)$ pairs from the data buffer by first sampling $(y, x)$ pairs and then applying the visual encoder $E$ of $\pi_{\theta}$ to observations. The optimization of $Q_{\psi}$ can be framed as follow:
\begin{equation}\label{discriminator_objective}
\begin{aligned}
\psi^* = \arg \min_{\psi} \{\frac{1}{\vert \mathcal{B}_{\text{t}} \vert}  &\sum_{(y, x) \sim \mathcal{B}_{\text{t}}}
\log (1 - Q_{\psi}(E_{\phi}(y), x)) \\ &+ 
\frac{1}{\vert \mathcal{B}_{\text{s}} \vert} \sum_{(y, x) \sim \mathcal{B}_{\text{s}}}
\ \log Q_{\psi}(E_{\phi}(y), x) \} 
\end{aligned}
\end{equation}

\begin{proposition} \label{KL_divergence estimation proposition}
Given a discriminator $Q_{\psi}$ trained based on \eqref{discriminator_objective}, we have
\begin{align*}
& \mathbb{E}_{p_s(x \mid \pi_\theta)}\!\left[
    d_{\mathrm{KL}}\!\left(
    p_s(l \mid x, \pi_\theta) \,\big\|\, p_t(l \mid x, \pi_\theta)
    \right)
    \right] \\ \approx&
    \frac{1}{\vert \mathcal{B}_{\text{s}} \vert} \sum_{(y, x) \sim \mathcal{B}_{\text{s}}}
    \log \frac{Q_{\psi^*}(E_{\phi}(y), x)}{1 - Q_{\psi^*}(E_{\phi}(y), x)} \frac{\widehat{p_{\mathcal{B}_t}(x)}}{\widehat{p_{\mathcal{B}_s}(x)}}
\end{align*}
where $\widehat{p_{\mathcal{B}_t}(x)}$ and $\widehat{p_{\mathcal{B}_s}(x)}$ are some functions approximating the distributions of $x$ in $\mathcal{B}_s$ and $\mathcal{B}_t$. $\vert \mathcal{B}_{\text{s}} \vert$ represents the cardinality of the source-domain dataset.
\end{proposition}

To understand why this proposition is algorithmically sensible, recall that a discriminator trained following the \eqref{discriminator_objective} is approximating:
\begin{align*}
Q_{\psi^*}(l, \; x) \approx \frac{p_{\mathcal{B}s}(l, x)}{p_{\mathcal{B}s}(l, x) + p_{\mathcal{B}t}(l, x)}
\end{align*}
where $p_{\mathcal{B}s}(l, x)$ and $p_{\mathcal{B}t}(l, x)$ are the distributions underlying the data buffers $\mathcal{B}_s$ and $\mathcal{B}_t$ respectively \cite{gan2014}. Then we can have the following derivation:
\begin{align*}
& \frac{1}{\lVert \mathcal{B}_{\text{s}} \rVert} \sum_{(y, x) \sim \mathcal{B}_{\text{s}}}
    \log \frac{Q_{\psi^*}(E_{\phi}(y), x)}{1 - Q_{\psi^*}(E_{\phi}(y), x)} \frac{\widehat{p_{\mathcal{B}_t}(x)}}{\widehat{p_{\mathcal{B}_s}(x)}} \\
&\approx \mathbb{E}_{p_{\mathcal{B}s}(l, x)}[\log \frac{Q_{\psi^*}(l, x)}{1 - Q_{\psi^*}(l, x)} \frac{p_{\mathcal{B}_t}(x)}{p_{\mathcal{B}_s}(x)}] \\
&\approx \mathbb{E}_{p_{\mathcal{B}s}(l, x)}[\log \frac{p_{\mathcal{B}s}(l, x)}{p_{\mathcal{B}_t}(l, x)} \frac{p_{\mathcal{B}_t}(x)}{p_{\mathcal{B}_s}(x)}] \\
&\approx \mathbb{E}_{{p_s}(l, x \mid \pi_\theta)}[\log \frac{p_{s}(l \mid x, \pi_\theta)}{p_{t}(l \mid x, \pi_\theta)}] \\
&= \mathbb{E}_{p_s(x \mid \pi_\theta)}\!\left[
    d_{\mathrm{KL}}\!\left(
    p_s(l \mid x, \pi_\theta) \,\big\|\, p_t(l \mid x, \pi_\theta)
    \right)
    \right]
\end{align*}

In this work, $Q_{\psi}$ is implemented as a two-layer neural network with $l$ and $x$ concatenated as inputs. $\widehat{p_{\mathcal{B}_t}(x)}$ and $\widehat{p_{\mathcal{B}_s}(x)}$ are implemented as two independent Gaussian Kernel Estimators fitted with data from $\mathcal{B}_t$ and $\mathcal{B}_s$ respectively.

\subsection{Adversarial Learning For Policy Improvements}
We now introduce State-Conditional Adversarial Learning (SCAL) to solve for objective \eqref{imitation_learning_surrogate}. Per propositions \ref{KL-divergence surrogate proposition} and \ref{KL_divergence estimation proposition}, the original intractable objective \eqref{imitation_learning_surrogate} can be optimized by
\begin{align*}
\label{tractable_optimization_problem}
\theta^* &= \arg \min_{\theta} \big\{
\mathcal{J}_{\text{s}}(\theta) + \lambda \; \mathcal{J}_{\text{adv}}(\theta)
\big\},
\end{align*}
where
\begin{align*}
\mathcal{J}_{\text{s}}(\theta) &= 
\mathbb{E}_{(y, x, u^\star) \sim \mathcal{B}_s}
\left[\mathcal{J}(\pi_{\theta}(y), u^*)\right]
\end{align*}
is the source-domain on-policy loss, and 
\begin{align*}
\mathcal{J}_{\text{adv}}(\theta) &=
\mathbb{E}_{(y, x, \cdot) \sim \mathcal{B}_s} \left[
    \log \frac{Q_{\psi^*}(E_{\phi}(y), x)}{1 - Q_{\psi^*}(E_{\phi}(y), x)} \frac{\widehat{p_{\mathcal{B}_t}(x)}}{\widehat{p_{\mathcal{B}_s}(x)}}\right]
\end{align*}
is the domain confusion loss,
and $Q_{\psi*}$ is optimized based on \eqref{discriminator_objective}. Note that $\mathcal{J}_s(\theta)$ can be redefined following other more advanced IL pipelines using the data from $\mathcal{B}_s$ or recollecting data from the source domain. Following the convention of adversarial training, the discriminator and the agent in our implementation are trained iteratively to preserve the expressivity of the discriminator.

\begin{figure}[htbp]
    \centering
    \begin{subfigure}[b]{0.45\linewidth}
        \centering
        \includegraphics[width=1.0\linewidth]{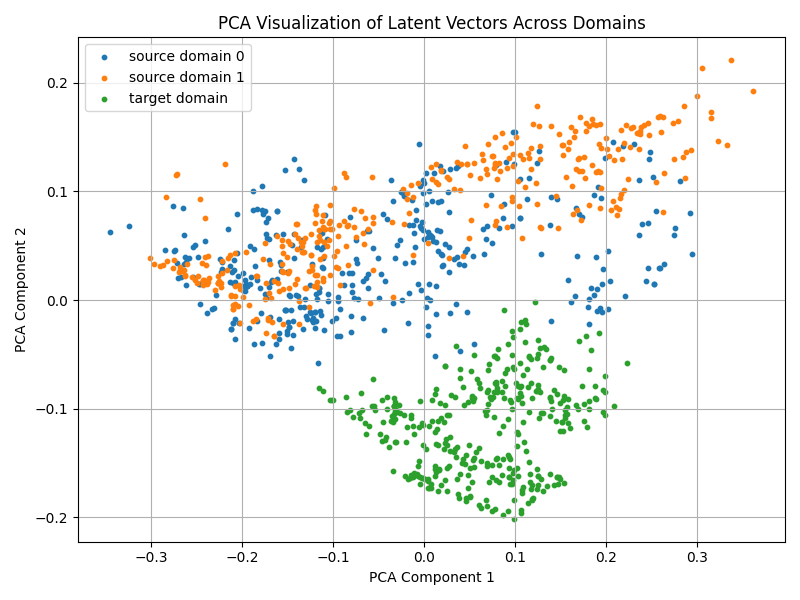}
        \label{fig:pca_before}
    \end{subfigure}
    \hfill
    \begin{subfigure}[b]{0.45\linewidth}
        \centering
        \includegraphics[width=1.0\linewidth]{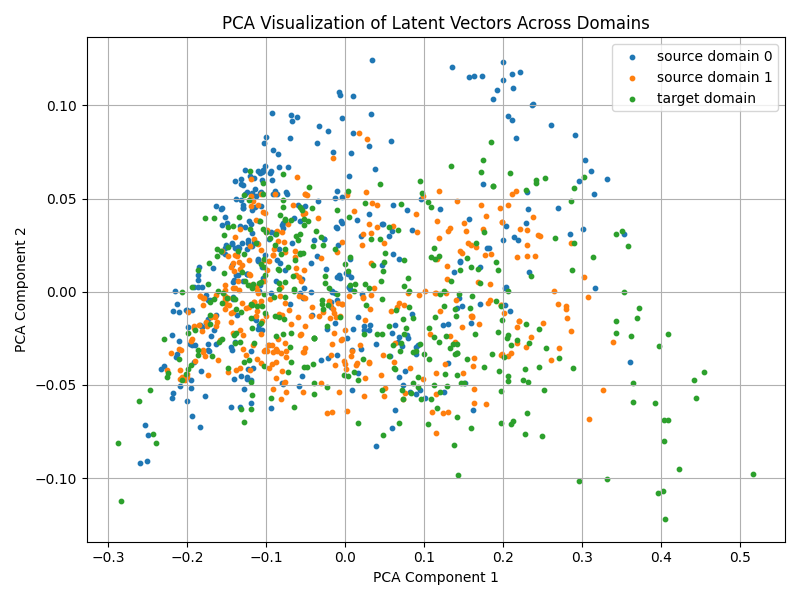}
        \label{fig:pca_after}
    \end{subfigure}
    \caption{PCA visualization of latent space with (left) and without (right) SCAL. The latent vectors presented are sampled from the same path-tracking trajectory.}
    \label{fig:pca_combined}
\end{figure}

\begin{algorithm}[t]
\caption{State-Conditional Adversarial Learning}
\label{alg:adv_policy_improvement}
\begin{algorithmic}[1]
\Require Source buffer $\mathcal{B}_s$, target buffer $\mathcal{B}_t$,
         initial parameters $\theta$, trade-off $\lambda$
\State Fit $\widehat{p_{\mathcal{B}_s}}(x)$ and $\widehat{p_{\mathcal{B}_t}}(x)$
       using the $x$-marginals of $\mathcal{B}_s$ and $\mathcal{B}_t$
\While{not converged}
    \State \textbf{// Update discriminator $Q_\psi$}
    \For{$k = 1, \dots, K_{\mathrm{disc}}$}
        \State Sample minibatch
            $\{(y_i^t, x_i^t)\}$ from $\mathcal{B}_t$
            and $\{(y_j^s, x_j^s)\}$ from $\mathcal{B}_s$
        \State Compute latents
            $l_i^t = E_{\phi}(y_i^t)$ and $l_j^s = E_{\phi}(y_j^s)$
        \State Compute discriminator loss
        \Statex \hspace{\algorithmicindent}
        $\mathcal{J}_{\mathrm{adv}}
        = - \frac{1}{|\mathcal{B}_t|}
           \sum_i \log\bigl(1 - Q_\psi(l_i^t, x_i^t)\bigr)
          - \frac{1}{|\mathcal{B}_s|}
           \sum_j \log Q_\psi(l_j^s, x_j^s)$
        \State Update $\psi \leftarrow
            \psi - \eta_\psi \nabla_\psi \mathcal{J}_{\mathrm{adv}}$
    \EndFor

    \State \textbf{// Update policy and encoder $(\theta)$}
    \State Collect $\mathcal{B}_s$ following the DAgger pipeline
    \State Sample minibatch
        $\{(y_j^s, u^{*}_j, \, x_j^s)\}$ from $\mathcal{B}_s$
    
    \State Compute 
    $\mathcal{J}_{\text{s}}$ 
    \State For each $(y_j^s, x_j^s)$, compute
        $l_j^s = E_{\phi}(y_j^s)$ and
        \Statex \hspace{\algorithmicindent}
        $w_j =
        \log \frac{Q_\psi(l_j^s, x_j^s)}
                 {1 - Q_\psi(l_j^s, x_j^s)}
        \frac{\widehat{p_{\mathcal{B}_t}}(x_j^s)}
             {\widehat{p_{\mathcal{B}_s}}(x_j^s)}$
    \State Compute adversarial loss
        $\mathcal{J}_{\mathrm{adv}}
        = \frac{1}{\vert \mathcal{B}_s \vert} \sum_j w_j$
    \State Compute total loss
        $\mathcal{J}_{\mathrm{total}}
        = \mathcal{J}_{\mathrm{s}}
          + \lambda \mathcal{J}_{\mathrm{adv}}$
    \State Update $\theta \leftarrow
        \theta - \eta_\theta
        \nabla_\theta \mathcal{J}_{\mathrm{total}}$
\EndWhile
\end{algorithmic}
\end{algorithm}


\section{Experiment}\label{sec:experiments}
In this section, our empirical analysis aims to verify: (A) the validity of the optimization surrogate in proposed by our theoretical analysis; (B) the sample efficiency of our method in an ablation study of various target domain data size and distribution; (C) 
the effectiveness of off-policy transfer in a low-speed-to-high-speed transfer learning experiment.

We base all of our experiment designs on Berkeley Autonomous Racing Car simulation environment (BARC) \citep{cao2025simple}, a racing simulator based-on CARLA. The source-domain system and target domain systems are custom environments with the same track geometry but drastically different visual appearances.

\begin{remark}
    In our experiments, the vehicle dynamics are kept identical between the source and target environments to isolate and highlight the effect of the proposed approach.
\end{remark}

\begin{figure}[htbp]
    \centering
    \begin{subfigure}[b]{0.45\linewidth}
        \centering
        \includegraphics[width=1.0\linewidth]{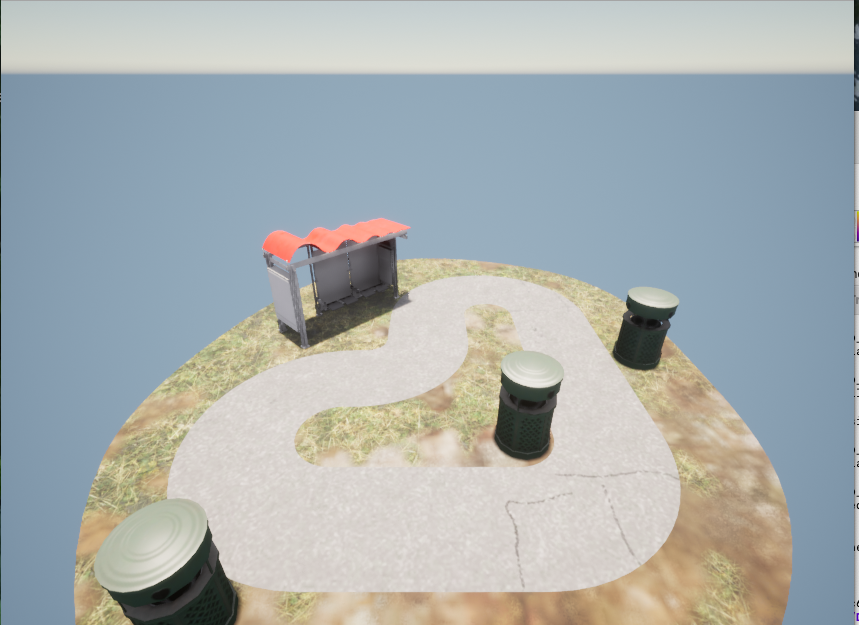}
        \label{fig:domain_demo1}
    \end{subfigure}
    \hfill
    \begin{subfigure}[b]{0.45\linewidth}
        \centering
        \includegraphics[width=1.0\linewidth]{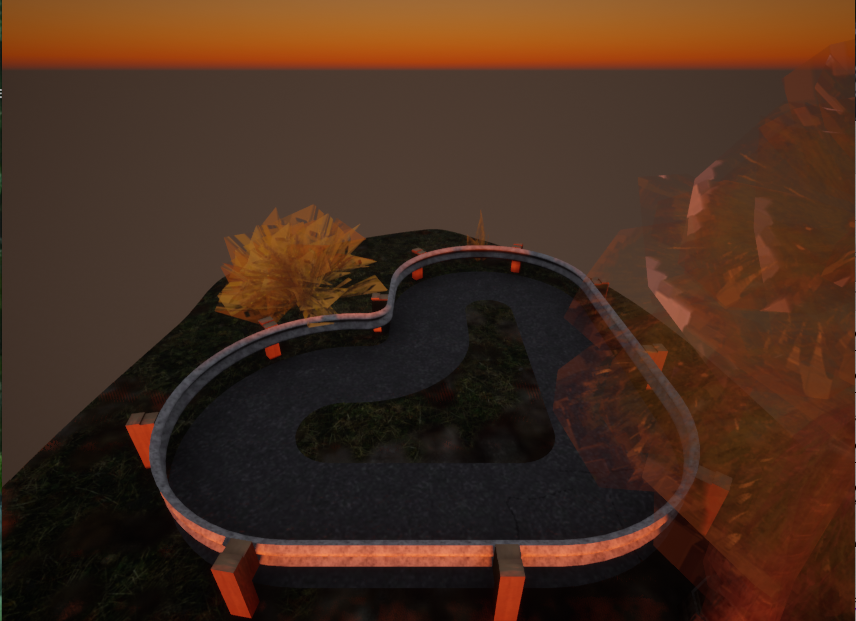}
        \label{fig:domain_demo2}
    \end{subfigure}
    \caption{Two example domains in our experiments with the same track shape but drastically different visual characters.}
    \label{fig:environments}
\end{figure}

\subsection{Upper-bound Empirical Validation} \label{Off-Policy Evaluation Study}
First, we empirically validate the relationship between $\mathcal{J}_t(\theta)$ in \eqref{imitation_learning_surrogate} and our derived upper bound \[
\mathcal{J}_s(\theta) \;+\;
\mathbb{E}_{p_s(x \mid \pi_\theta)}\!\left[
    d_{\mathrm{KL}}\!\left(
        p_s(l \mid x, \pi_\theta)
        \,\big\|\,
        p_t(l \mid x, \pi_\theta)
    \right)
\right],
\] 
which is the proposed surrogate from our theoretical analysis.

We trained 
20 agents in different source domains until the imitation loss $\mathcal{J}_s({\theta_i})$ converge to near zero for all agents. The supervision for all these agents comes from the same black-boxed PID controller. The value of the State-Conditional KL-divergence \[\mathbb{E}_{p_s(x \mid \pi_{\theta_i})}\!\left[
    d_{\mathrm{KL}}\!\left(
        p_s(l \mid x, \pi_{\theta_i})
        \,\big\|\,
        p_t(l \mid x, \pi_{\theta_i})
    \right)
\right]\] is then calculated based on Proposition \ref{KL_divergence estimation proposition} for each agent. Then, without any further training, $\mathcal{J}_t(\theta_i)$ is estimated from the on-policy imitation loss and the on-policy trajectory length in the target domain.
 
As illustrated by figure \ref{fig:OPE_visualization}, with each agent achieving near-zero source domain on-policy loss, a strong correlation is demonstrated between the estimated state-conditional KL divergence and the agent's on-policy behavior in the target domain.

\begin{figure}[htbp]
    \centering
    \includegraphics[width=1.0\linewidth]{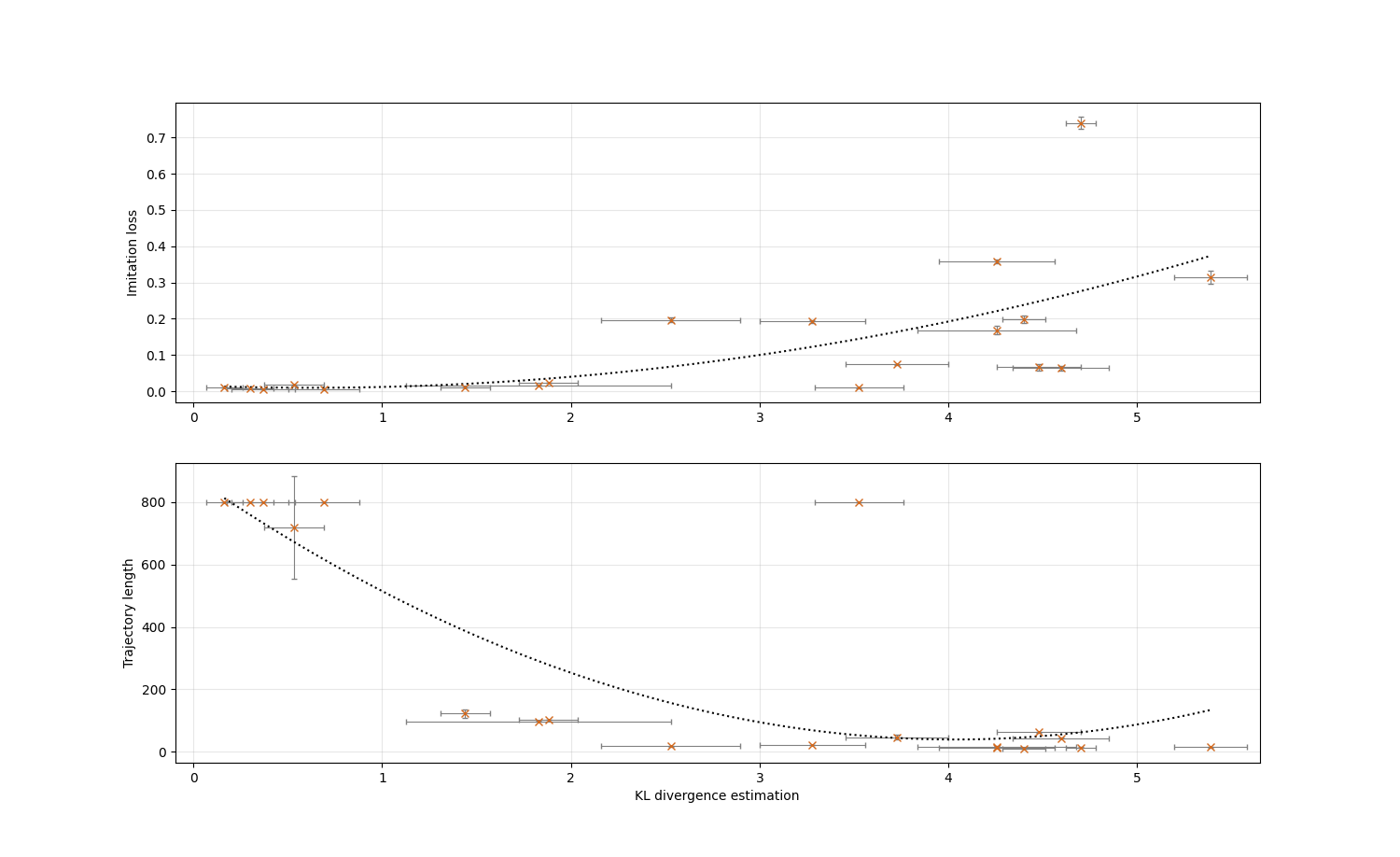}
    \caption{Correlation between estimated State-Conditional KL divergence and On-policy target domain metric.}
    \label{fig:OPE_visualization}
\end{figure}

\subsection{Ablation Study of Target-Domain Data Efficiency} \label{Distributional-shift Study}

We are interested in the sample efficiency of our framework under different distributions of $\mathcal{B}_t$. We choose $\pi_{\beta}$ to be a path tracking PID controller. We predefine three different state space distributions to collect $\mathcal{B}_t$. Following each distribution, we collected $\mathcal{B}_s$ with the following sizes: ${2048, 1024, 512, 256, 213, 170, 128}$. For each $\mathcal{B}_t$, five independent trials of training are conducted following our pipeline, and we compare the maximum collision-free trajectory length achieved by the agent in the target domain.

As a reference for performance, we include DAgger \citep{ross11noregret} with direct access to full information in the target domain as a baseline. 
More specifically, we directly train this baseline under a fully supervised and online setting with access to expert demonstrations in the target domain. 
We use this baseline as an approximation for the upper bound for the sample efficiency for off-policy transfer learning algorithms.

As illustrated in Figure \ref{fig:datasize_exp}, despite relying solely on an offline buffer without expert supervision, our approach achieves comparable or even superior trajectory lengths compared to DAgger under all three off-policy $\mathcal{B}_t$ distributions. Notably, SCAL maintains strong performance and stability even in low-data regimes (e.g., with only 256 target samples), demonstrating its competitive sample efficiency. 
In addition, the experimental result also shows the robustness of our methods to various distributions $\mathcal{B}_t$. 

\begin{figure}[htbp]
    \centering
    \includegraphics[width=1.0\linewidth]{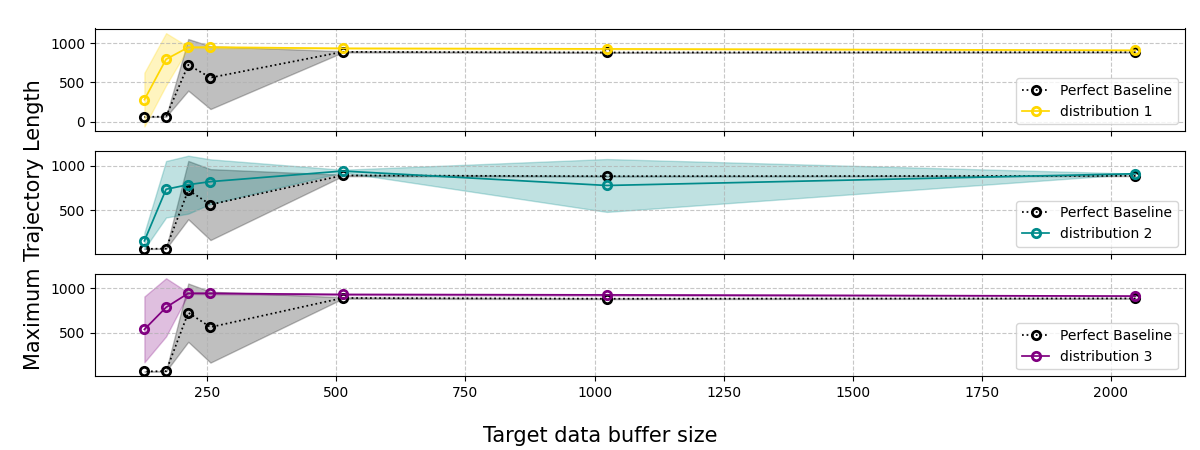}
    \caption{SCAL (yellow, blue, purple) compared with perfect baseline(black) under different $\mathcal{B}_s$ distributions. x-axis: Target-domain buffer size. y-axis: Maximum trajectory length achieved in the target domain. The shaded area represents variance. }
    \label{fig:datasize_exp}
\end{figure}

\begin{figure}[htbp]
    \centering
    \includegraphics[width=1.0\linewidth]{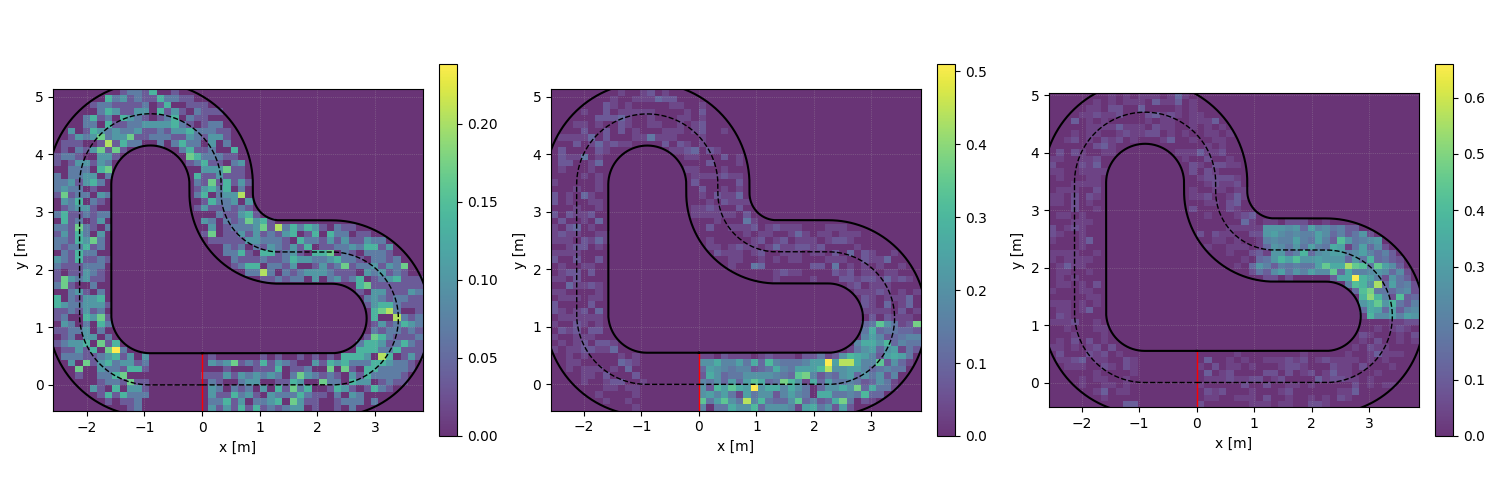}
    \caption{Three different target-domain off-policy sample distributions used in experiment \ref{Distributional-shift Study}. The brighter area stands for states sampled with higher frequency.}
    \label{fig:sample_distribution}
\end{figure}

\subsection{Low-Speed-to-High-Speed Transfer} \label{Low-Speed-to-High-Speed Transfer}
In this experiment, we empirically validate the effectiveness of our approach under off-policy data in a safety-critical high-speed racing task. 
The expert policy is a high-performing MPCC-Conv controller used in~\citep{cao2025simple}, while the target-domain dataset is collected using a conservative, low-speed PID controller. The discriminator $Q_\psi$ is conditioned on the full state $x$. The objective is to train $\pi_\theta$ to be able to race at high speed in both environments without supervision and on-policy data in the target domain.

Two key challenges in this task are: (1) a significant distributional discrepancy between $\mathcal{B}_s$ and $\mathcal{B}_t$, arising from the distinct trajectory characteristics of the MPCC-Conv and PID controllers; and (2) the heightened sensitivity of high-speed imitation learning to small prediction errors, which demands precise policy alignment. 

As shown in Fig.~\ref{fig:demonstration of low-speed-to-high-speed transfer}, the proposed framework can preserve high imitation accuracy and stable performance despite substantial distribution shifts.

\begin{figure}[htbp]
    \centering
    \includegraphics[width=1.0\linewidth]{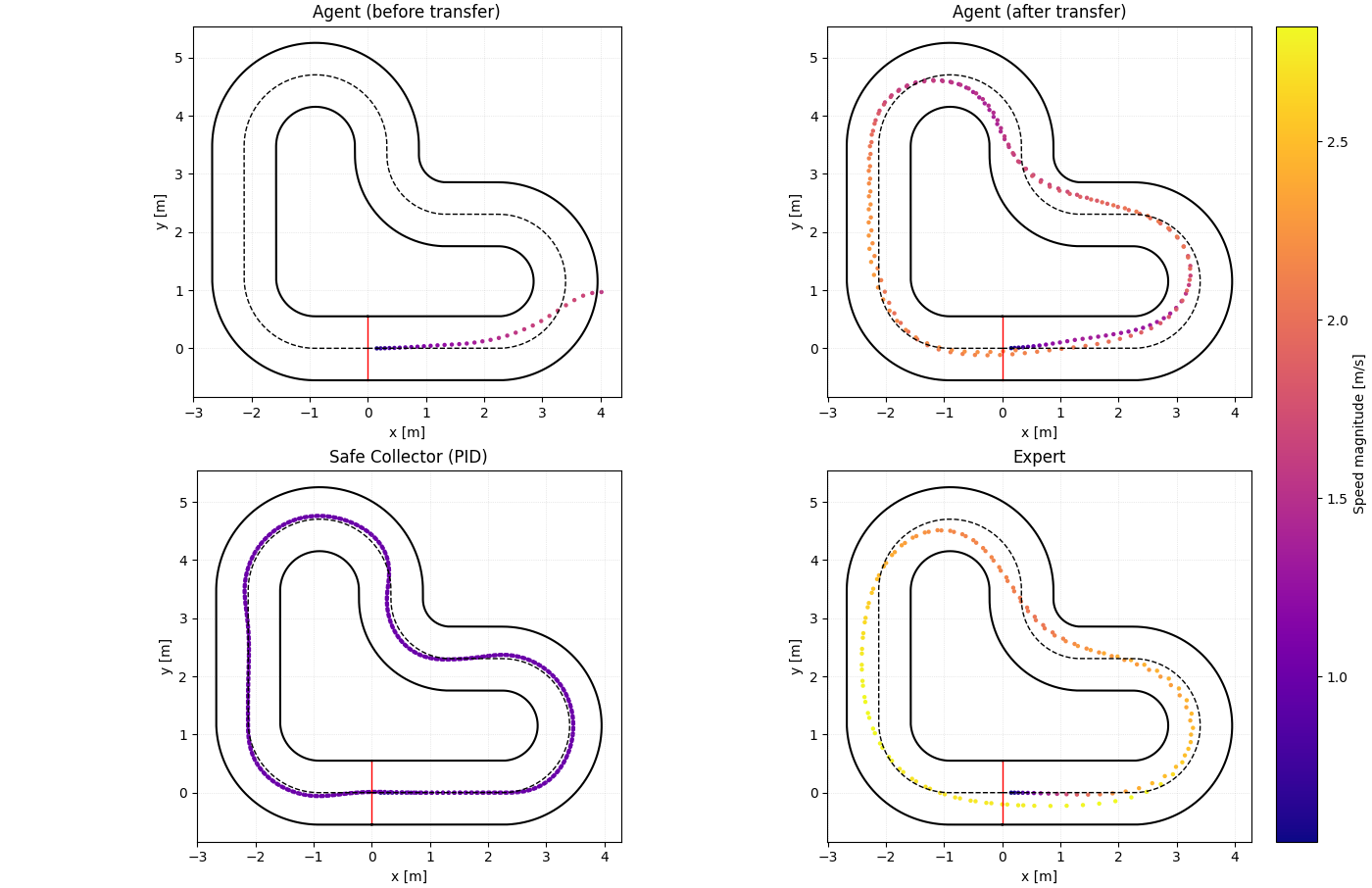}
    \caption{low-speed-to-high-speed transfer. Successful transfer of high-speed policy (top right) with off-policy low-speed data(down left).}
    \label{fig:demonstration of low-speed-to-high-speed transfer}
\end{figure}

\section{Future Works}
As future work, we plan to extend our State-Conditional Adversarial Learning (SCAL) framework to domains with higher-dimensional state and action spaces, such as agile locomotion and dexterous manipulation, to further validate the scalability of our approach.
In addition, we aim to extend this framework to domain transfer tasks with dynamics discrepancies.                                                          

\bibliographystyle{plainnat}  
\bibliography{ref}         

\section{Appendix}

\subsection{proof appendix for Lemma \ref{lemma_one}}
This section aims to proof the correctness of \ref{lemma_one}

\begin{proof}
Note that the expert can be viewed as a history-dependent policy
\[
\pi_{\beta} : \bigcup_{k\ge 0} \mathcal{X}^{k+1} \to \mathcal{U}, 
\qquad
u^*_k = \pi_{\beta}(x_{0:k}),
\]
For domain $d\in\{s,t\}$, the discounted imitation loss can be framed as
\[
\mathcal{J}_d(\theta)
= \sum_{k=0}^{\infty} (1-\gamma)\gamma^k\mathbb{E}_{p^k(y_k, u^*_k \mid \pi_{\theta})}\!\left[
  \,\mathcal{L}\!\left(\pi_{\theta}(y_k),\, u^*_k\right)\right],
\]
where the expert action is $u^*_k = \pi_{\beta}(x_{0:k})$.

We use the following generative model in domain $d$:
\begin{align*}
x_0 \sim p_0,\qquad
y_k \sim e_d(\cdot\mid x_k),\qquad
l_k = E(y_k),\qquad \\
u_k = D(l_k),\qquad
x_{k+1} = f(x_k,u_k),\qquad
u^*_k = \pi_{\beta}(x_{0:k}).
\end{align*}

At time step $k$, the imitation loss term can be written as
\[
\mathcal{L}\!\left(\pi_{\theta}(y_k), u^*_k\right)
= \mathcal{L}\!\left(D(l_k), \pi_{\beta}(x_{0:k})\right).
\]

Conditioned on the prefix $x_{0:k}$, the expert action is deterministic, while
$l_k$ depends only on $x_k$. Thus
\[
p_d(l_k \mid x_{0:k},\pi_{\theta})
= p_d(l_k \mid x_k,\pi_{\theta}).
\]
Define the per-prefix surrogate loss
\[
g_d(x_{0:k})
:= \mathbb{E}_{l_k \sim p_d(l\mid x_k,\pi_{\theta})}
   \!\left[\mathcal{L}\!\left(D(l_k), \pi_{\beta}(x_{0:k})\right)\right].
\]
Then
\[
\mathcal{J}_d(\theta)
= \sum_{k=0}^{\infty} (1-\gamma)\gamma^k\,
  \mathbb{E}_{x_{0:k}\sim p_d(x_{0:k}\mid\pi_{\theta})}
  \!\left[g_d(x_{0:k})\right].
\]

Using the definition of alignment, the encoder $E_{\phi}$ of the agent can induce
\[
p_s(l\mid x, \pi_{\theta}) = p_s(l\mid x, E_{\phi}) = p_t(l\mid x, E_{\phi}) = p_t(l\mid x, \pi_{\theta}),
\]
we obtain, for any prefix $x_{0:k}$,
\begin{align*}
g_s(x_{0:k})
&= \mathbb{E}_{l_k \sim p_s(l\mid x_k,\pi_{\theta})}
   \!\left[\mathcal{L}\!\left(D(l_k), \pi_{\beta}(x_{0:k})\right)\right] \\
&= \mathbb{E}_{l_k \sim p_t(l\mid x_k,\pi_{\theta})}
   \!\left[\mathcal{L}\!\left(D(l_k), \pi_{\beta}(x_{0:k})\right)\right] \\
&=: g(x_{0:k}).
\end{align*}
Thus $g_s(x_{0:k}) = g_t(x_{0:k}) = g(x_{0:k})$ for all $x_{0:k}$.

For any $x$,
\begin{align*}
p_s(u \mid x,\pi_{\theta})
&= \int \delta\!\left(u - D(l)\right)\,
        p_s(l\mid x,\pi_{\theta})\, dl \\
&= \int \delta\!\left(u - D(l)\right)\,
        p_t(l\mid x,\pi_{\theta})\, dl \\
&= p_t(u \mid x,\pi_{\theta}).
\end{align*}

Since the two domains share $p_0$ and the dynamics $x_{k+1}=f(x_k,u_k)$,
the Markov chains induced by $\pi_{\theta}$ are identical.
By induction on $k$, this yields
\[
p_s(x_{0:k}\mid\pi_{\theta})
= p_t(x_{0:k}\mid\pi_{\theta}),\qquad \forall k.
\]

Combining $g_s(x_{0:k})=g_t(x_{0:k})=g(x_{0:k})$ with
$p_s(x_{0:k}\mid\pi_{\theta}) = p_t(x_{0:k}\mid\pi_{\theta})$, we obtain
\begin{align*}
\mathcal{J}_s(\theta)
&= \sum_{k=0}^{\infty} (1-\gamma)\gamma^k\,
   \mathbb{E}_{x_{0:k}\sim p_s(x_{0:k}\mid\pi_{\theta})}
   \!\left[g(x_{0:k})\right] \\
&= \sum_{k=0}^{\infty} (1-\gamma)\gamma^k\,
   \mathbb{E}_{x_{0:k}\sim p_t(x_{0:k}\mid\pi_{\theta})}
   \!\left[g(x_{0:k})\right] \\
&= \mathcal{J}_t(\theta).
\end{align*}
\end{proof}

\subsection{Proof Appendix for Theorem \ref{theorem_one}}
For the simplicity of notations, in this proof, we use $p(\cdot)$ as a default shorthand for $p(\cdot \mid \pi_{\theta})$ all the visitation distributions. Similarly, we use $p(\cdot \mid x)$ as a short hand for $p(\cdot \mid x, E_{\phi})$. That means, by default, we assume every distribution in this proof is conditioned under the agent $\pi_{\theta}$. Consider the following derivation

\begin{align*}
&|\mathcal{J}_t(\theta) - \mathcal{J}_s(\theta)| \\
&= \Big|
\mathbb{E}_{p_s(y, u^{*})}\!\big[\mathcal{L}(\pi_{\theta}(y), u^{*}) \big]
- \mathbb{E}_{p_t(y, u^{*})}\!\big[\mathcal{L}(\pi_{\theta}(y), u^{*}) \big]
\Big| \\[4pt]
&= \left| 
\int \int \mathcal{L}(\pi_{\theta}(y), u^*)
\big(p_s(y, u^*) - p_t(y, u^*)\big)\ dy \; du^*
\right| \\[4pt]
&\le \alpha \int \int \big|p_s(y, u^*) - p_t(y, u^*)\big|\, dy \; du^* \\[4pt] \tag*{(by the definition of Total Variance)}
&\leq 2\alpha \, d_{\mathrm{TV}}\!\big(p_s(y, u^*),\,p_t(y, u^*)\big) \\[4pt] \tag*{(by Pinsker's Inequality)}
&\leq 2\alpha \sqrt{\tfrac{1}{2}\, d_{\mathrm{KL}}\!\big(p_s(y, u^*)\,\|\,p_t(y, u^*)\big)} \\[4pt]
&= \alpha \, \sqrt{2\, d_{\mathrm{KL}}\!\big(p_s(y, u^*)\,\|\,p_t(y ,u^*)\big)}. \\
&\text{where} \quad \alpha = \sup_{y \in \mathcal{Y}, \, u^* \in \mathcal{U}}\mathcal{L}\big(\pi_{\theta}(y), u^*\big) 
\end{align*}

Then the problem suffices to find an upper bound for $d_{\mathrm{KL}}\!\big(p_s(y,u^*)\,\|\,p_t(y,u^*)\big)$. Note that by the definition $\gamma$-discounted distribution and the convexity of $d_{\text{KL}}$. We can obtain the following:
\begin{align*}
&d_{\mathrm{KL}}\!\big(p_s(y,u^*)\,\|\,p_t(y,u^*)\big) \\
&\leq (1- \gamma) \sum_{k = 0}^{\infty}\gamma^k d_{\mathrm{KL}}\!\big(p^k_s(y_k,u^*_k)\,\|\,p^k_t(y_k,u^*_k)\big) \\
& \leq (1- \gamma) \sum_{k = 0}^{\infty}\gamma^k d_{\mathrm{KL}}\!\big(p^k_s(y_k,x_{0:k})\,\|\,p^k_t(y_k,x_{0:k})\big)
\end{align*}
To understand why the second inequality holds, recall that the joint distributions $p_d^k(y, u^*) \; \forall d \in \{ s, t\}$ are pushed-forward distributions obtained by applying the function $u^* = \pi_{\beta}(x_{0:k})$ to the distributions $p_d^k(y, x_{0:k}) \; \forall d \in \{ s, t\}$. Thus, the second inequality holds by the Data Process Theorem.

Since $p^k_s(y_k,x_{0:k}) = e_s(y_k \mid x_k) \; \cdot \; p^k_s(x_{0:k})$ and $p^k_t(y_k,x_{0:k}) = e_t(y_k \mid x_k) \; \cdot \; p^k_t(x_{0:k})$ by the chain rule of KL divergence, we will have the following:

\begin{align*}
&d_{\mathrm{KL}}\!\big(p^k_s(y,x_{0:k})\,\|\,p^k_t(y,x_{0:k})\big) \\
&= \mathbb{E}_{p_s^k(x)}[d_{\text{KL}}(e_s(\cdot \mid x) \| e_t(\cdot \mid x))] + d_{\text{KL}}(p_s(x_{0:k}) \|p_t(x_{0:k}))
\end{align*}

Together, we will have the following upper bound for $d_{\mathrm{KL}}\!\big(p_s(y,u^*)\,\|\,p_t(y,u^*)\big)$.

\begin{align*}
&d_{\mathrm{KL}}\!\big(p_s(y,u^*)\,\|\,p_t(y,u^*)\big) \\
&\leq (1- \gamma) \sum_{k = 0}^{\infty}\gamma^k \{ \mathbb{E}_{p_s^k(x)}[d_{\text{KL}}(e_s(\cdot \mid x) \| e_t(\cdot \mid x))] \\ 
&+ d_{\text{KL}}(p_s(x_{0:k}) \| p_t(x_{0:k}))\} \\
&= \underbrace{(1- \gamma) \sum_{k = 0}^{\infty}\gamma^k \{ \mathbb{E}_{p_s^k(x)}[d_{\text{KL}}(e_s(\cdot \mid x) \| e_t(\cdot \mid x))] \}}_{\text{Part A}} \\
&+ \underbrace{(1- \gamma) \sum_{k = 0}^{\infty}\gamma^k \{d_{\text{KL}}(p_s(x_{0:k}) \|p_t(x_{0:k}))\}}_{\text{part B}}
\end{align*}
The problem now suffices to find compact bound notations for part A and part B.

\paragraph{Part A}
Based on the definition of $\gamma$-discounted distribution, part A can be re-formulated:
\begin{align*}
&(1- \gamma) \sum_{k = 0}^{\infty}\gamma^k \{ \mathbb{E}_{p_s^k(x)}[d_{\text{KL}}(e_s(\cdot \mid x) \| e_t(\cdot \mid x))] \} \\
&= \mathbb{E}_{p_s(x)}[d_{\text{KL}}(e_s(\cdot \mid x) \| e_t(\cdot \mid x))]
\end{align*}
Note that this term depicts the distributional discrepancy determined by the observation model, which is mostly not optimizable. In the following proof, we will refer this term as a constant $\sigma$.

\paragraph{Part B}
Consider the following:
\begin{align*}
&d_{\text{KL}}(p_s(x_{0 : k}) \; \| \; p_t(x_{0:k})) \\
&= d_{\text{KL}}(p^{0}_s(x) \; \| \; p^{0}_t(x)) \; \\ &+ \sum_{i = 0}^{k - 1} \mathbb{E}_{p^i_s(x_i)}[d_{\text{KL}}(p^i_s(x_{i+1} \mid x_i) \; \| \; p^i_t(x_{i+1} \mid x_i)] \\
&= \sum_{i = 0}^{k - 1} \mathbb{E}_{p^i_s(x_i)}[d_{\text{KL}}(p^i_s(x_{i+1} \mid x_i) \; \| \; p^i_t(x_{i+1} \mid x_i)]\\
\end{align*}

The first equality is by the chain rule of KL divergence. The second equality is by the assumption that both domains share the same initial state distribution. 

Note that $p^i_t(x_{i+1} \mid x_i) = p^i_t(x_{i + 1} \mid l) \; \cdot \; p^i_t(l \mid x_i) = \delta(x_{i+1}, D_\theta(l))p^i_t(l \mid x_i)$, where $\delta$ is Kronecker function. Thus, $p^i_t(x_{i+1} \mid x_i)$ can be viewed as a distribution obtained by applying channel $\delta(x_{i+1}, D_\theta(l))$ to the distribution $p_t(l \mid x_i)$. Then, by the Data Process Theorem, we will have $d_{\text{KL}}(p^i_s(x_{i+1} \mid x_i) \; \| \; p^i_t(x_{i+1} \mid x_i) \; \leq \; d_{\text{KL}}(p^i_s(l \mid x_i) \; \| \; p^i_t(l \mid x_i) \; \forall i, \; x_i$. With this fact, we can further refine the upper bound for $d_{\text{KL}}(p_s(x_{0:k}) \; \| \; p_t(x_{0:k}))$:
\begin{align*}
d_{\text{KL}}(p_s(x_{0:k}) \; \| \; p_t(x_{0:k)}])) \\
\leq \sum_{i = 0}^{k - 1} \mathbb{E}_{p^i_s(x)}[d_{\text{KL}}(p^i_s(l \mid x) \; \| \; p^i_t(l \mid x)]
\end{align*}
 Now, plug this back to the expression of Part B, we will get:
 \begin{align*}
&(1- \gamma) \sum_{k = 0}^{\infty}\gamma^k \{d_{\text{KL}}(p_s(x_{0:k}) \|p_t(x_{0:k}))\} \\
&\leq (1 - \gamma) \sum_{t \geq 0}^\infty \gamma^t \;\sum_{i = 0}^{t - 1} \mathbb{E}_{p^i_s(x)}[d_{\text{KL}}(p^i_s(l \mid x) \; \| \; p^i_t(l \mid x) \\
&= (1 - \gamma) \;\sum_{i = 0}^{\infty} \mathbb{E}_{p^i_s(x)}[d_{\text{KL}}(p^i_s(l \mid x) \; \| \; p^i_t(l \mid x) \sum_{t \geq i+1}^\infty \gamma^t \\
&= (1 - \gamma) \;\sum_{i = 0}^{\infty} \mathbb{E}_{p^i_s(x)}[d_{\text{KL}}(p^i_s(l \mid x) \; \| \; p^i_t(l \mid x)]\frac{\gamma^{i+1}}{1-\gamma} \\
&= \sum_{i = 0}^{\infty} \gamma^{i+1} \mathbb{E}_{p^i_s(x)}[d_{\text{KL}}(p^i_s(l \mid x) \; \| \; p^i_t(l \mid x)] \\
&= \frac{\gamma}{1-\gamma} \mathbb{E}_{p_s(x)}[d_{\text{KL}}(p_s(l \mid x) \; \| \; p_t(l \mid x)]
\end{align*}

\paragraph{Conclusion}
Putting all the things together, we will get the full upper bound:
\begin{align*} \label{eq:upper_bound}
&\mathcal{J}_t(\theta) \;\leq\; 
\mathcal{J}_s(\theta)
+ \\ &\alpha \sqrt{ 
    \frac{2\gamma}{1 - \gamma} \, (
    \mathbb{E}_{p_s(x \mid \pi_{\theta})}[d_{\mathrm{KL}}\!\left(p_s(l \mid x, \pi_{\theta}) \,\big\|\, p_t(l \mid x, \pi_{\theta})\right)] 
    + \sigma )}
\end{align*}
where
\begin{itemize}
    \item $\sigma = \mathbb{E}_{p_s(x)}[d_{\mathrm{KL}}\!\big(e_s(\cdot \mid x)\,\|\, e_t(\cdot \mid x)\big)]$,
    \item $\alpha$ is the uniform bound over the loss function, with
    $
    \alpha = \sup_{y \in \mathcal{Y}, \; u^{*} \in \mathcal{U}}
    \mathcal{L}\big(\pi_{\theta}(y), u^{*} \big).
    $
\end{itemize}

\subsection{Mathematical Justification for the magnitude of $\sigma$ in Theorem \ref{eq:upper_bound}}

\label{sec:remark_justification}

In this section, we provide a mathematical justification for the claim that the inherent divergence $\sigma$ between the source and target observation models can be kept reasonably small through robust data preprocessing, such as input normalization.

Recall the definition of the expected conditional Kullback-Leibler (KL) divergence over the target state distribution:
\begin{equation} \label{eq:sigma_def}
    \sigma = \mathbb{E}_{p_t(x)} \left[ d_{\mathrm{KL}}\!\big(e_s(\cdot \mid x) \,\|\, e_t(\cdot \mid x)\big) \right].
\end{equation}

To analyze this term analytically, we assume that the observation models for both the source and target domains follow linear Gaussian distributions. Let $x \in \mathbb{R}^d$ be the underlying state, and let the observation distributions be parameterized as:
\begin{align}
    e_s(o \mid x) &= \mathcal{N}\left(W_s x + b_s, \Sigma_s\right), \\
    e_t(o \mid x) &= \mathcal{N}\left(W_t x + b_t, \Sigma_t\right),
\end{align}
where $W_s, W_t$ are the weight matrices, $b_s, b_t$ are the bias vectors, and $\Sigma_s, \Sigma_t$ are the positive definite covariance matrices for the source and target domains, respectively.

For two $k$-dimensional multivariate Gaussian distributions $\mathcal{N}_0(\mu_0, \Sigma_0)$ and $\mathcal{N}_1(\mu_1, \Sigma_1)$, the KL divergence is given by:
\begin{equation}
\begin{split}
    d_{\mathrm{KL}}(\mathcal{N}_0 \,\|\, \mathcal{N}_1) &= \frac{1}{2} \Biggl( \operatorname{tr}(\Sigma_1^{-1} \Sigma_0) - k + \ln\left(\frac{|\Sigma_1|}{|\Sigma_0|}\right) \\
    &\quad + (\mu_1 - \mu_0)^\top \Sigma_1^{-1} (\mu_1 - \mu_0) \Biggr).
\end{split}
\end{equation}

Substituting our linear Gaussian models into this formula, the conditional KL divergence for a given state $x$ is:
\begin{equation}
\begin{split}
    &d_{\mathrm{KL}}\!\big(e_s(\cdot \mid x) \,\|\, e_t(\cdot \mid x)\big)\\ &= \frac{1}{2} C_{\Sigma} + \frac{1}{2} \big( (W_t - W_s)x \\ &+ (b_t - b_s) \big)^\top \Sigma_t^{-1} \big( (W_t - W_s)x + (b_t - b_s) \big),
\end{split}
\end{equation}
where $C_{\Sigma} = \operatorname{tr}(\Sigma_t^{-1} \Sigma_s) - k + \ln\left(\frac{|\Sigma_t|}{|\Sigma_s|}\right)$ encapsulates the divergence purely due to covariance mismatch.

Let $\Delta W = W_t - W_s$ and $\Delta b = b_t - b_s$. Taking the expectation over the target state distribution $x \sim p_t(x)$ as per Equation~\ref{eq:sigma_def}, we obtain the expanded expression for $\sigma$:
\begin{equation} \label{eq:sigma_expanded}
\begin{split}
    \sigma &= \frac{1}{2} C_{\Sigma} + \frac{1}{2} \mathbb{E}_{p_t(x)} \Biggl[ x^\top \Delta W^\top \Sigma_t^{-1} \Delta W x \\
    &\quad + 2 \Delta b^\top \Sigma_t^{-1} \Delta W x + \Delta b^\top \Sigma_t^{-1} \Delta b \Biggr].
\end{split}
\end{equation}

Let $\mu_x = \mathbb{E}_{p_t(x)}[x]$ and $\Sigma_x = \mathbb{E}_{p_t(x)}[(x - \mu_x)(x - \mu_x)^\top]$ denote the mean and covariance of the states in the target domain. Using the property of expectations of quadratic forms (the trace trick), we can rewrite Equation~\ref{eq:sigma_expanded} exactly as:
\begin{equation} \label{eq:sigma_final}
\begin{split}
    \sigma &= \frac{1}{2} C_{\Sigma} + \frac{1}{2} \operatorname{tr}\left( \Delta W^\top \Sigma_t^{-1} \Delta W (\Sigma_x + \mu_x \mu_x^\top) \right) \\
    &\quad + \Delta b^\top \Sigma_t^{-1} \Delta W \mu_x + \frac{1}{2} \Delta b^\top \Sigma_t^{-1} \Delta b.
\end{split}
\end{equation}

\paragraph{The Role of Normalization}
Equation~\ref{eq:sigma_final} demonstrates exactly how robust normalization bounds $\sigma$. If we apply standard normalization (e.g., zero-mean, unit-variance normalization) to the input states $x$ prior to processing, we enforce that the data is zero-centered ($\mu_x \approx 0$) and has a constrained scale ($\Sigma_x \approx I$). Under these normalized conditions, the terms dependent on the state mean vanish, simplifying the bound to:
\begin{equation}
    \sigma \approx \frac{1}{2} C_{\Sigma} + \frac{1}{2} \operatorname{tr}\left( \Delta W^\top \Sigma_t^{-1} \Delta W \right) + \frac{1}{2} \Delta b^\top \Sigma_t^{-1} \Delta b.
\end{equation}

Furthermore, we can normalize the observation outputs. By standardizing the scales and removing baseline shifts of $y$, the linear transformation mismatch $\Delta W \to 0$ and the bias mismatch $\Delta b \to 0$. Similarly, aligning the observation variance drives $\Sigma_t^{-1}\Sigma_s \to I$, which makes the covariance penalty $C_{\Sigma} \to 0$. Therefore, appropriate ahead-of-time data normalization effectively controls the quadratic growth of the divergence, allowing $\sigma$ to be bounded to a reasonably small constant in practice.

\subsection{Implementation Details for Experiment}
\subsubsection{Agent Architecture Design}
The input to the agent consists of an RGB image \( y \in \mathbb{R}^{224 \times 224 \times 3} \) and the vehicle's velocity \( v \), forming the observation vector \( [y, v]^T \). $v$ is the velocity vector defined as $v = [v_{\text{long}}, v_{\text{tran}}]^T$, where $v_{\text{long}}$ is the longitudinal velocity and $v_{\text{tran}}$ is the lateral velocity in the vehicle's body frame. The output decisions by both the agent and the expert are $[u_a, u_{\text{steer}}]^T$, corresponding to throttle and steering control.
The visual encoder used in our framework is a ResNet-18, which maps the RGB image observation $y$ to a latent vector $l \in \mathbb{R}^{512}$. To balance the dimensionality between the latent vector and the velocity input, the velocity $v$ is first projected into a 16-dimensional space via a linear layer. This transformed velocity vector is then concatenated with the latent vector, resulting in a fused decision vector of dimension 528. This fused vector is passed through a single decision layer to get the output decision.

\subsubsection{Discriminator Architecture Design}
The discriminator is implemented as a two-layer multilayer perceptron (MLP), consisting of a linear layer $\mathbb{R}^{512} \rightarrow \mathbb{R}^{256}$, followed by a ReLU activation layer $\mathbb{R}^{256} \rightarrow \mathbb{R}^{256}$, and a final linear decision layer $\mathbb{R}^{256} \rightarrow \mathbb{R}^{1}$. The output logits of the discriminator are passed through a sigmoid function.

\subsubsection{Gaussian Kernel Estimators}
We implement $\widehat{p_{\mathcal{B}_t}(x)}$ and $\widehat{p_{\mathcal{B}_s}(x)}$ as two independent Gaussian Kernel Estimators fitted with data from $\mathcal{B}_t$ and $\mathcal{B}_s$ respectively. During adversarial transfer learning, they are fitted only once at the start of the training. Then, they are frozen and treated as two fixed weight functions for KL estimation.

\subsubsection{Learning Rates for Transfer Learning}
One main drawback of adversarial learning families lies in highly sensitive and in-robust learning rates. Finding the right learning rates for the discriminator and the agent usually requires huge efforts of hyper tuning. For different experiments we have done, the learning rate usually vary through a wide range. For the future researchers who want to implement this method, we highly encourage them to carefully tune the agent's and the discriminator's learning rates based on their own problem settings.

\subsubsection{State Definition}
The system state \( x \) is defined as:
$
x = [e_\psi, e_s, K(s_0), K(s_1), K(s_2)]^T
$
where \( e_\psi \) is the heading error in the Frenet frame, \( e_s \) is the  deviation from the reference centerline, and \( K(s_0), K(s_1), K(s_2) \) denote the curvatures of the reference trajectory at three discretely sampled Frenet coordinates ahead of the vehicle's current position.

\end{document}